\documentclass[preprint,12pt]{elsarticle}

\usepackage{amssymb}
\usepackage{amsmath}

\makeatletter
\def\ps@pprintTitle{%
  \let\@oddhead\@empty
  \let\@evenhead\@empty
  \def\@oddfoot{\reset@font\hfil\thepage\hfil}
  \let\@evenfoot\@oddfoot
}
\makeatother

\begin{document}

%
%
\newdefinition{definition}{Definition}
\newtheorem{theorem}{Theorem}
\newtheorem{lemma}[theorem]{Lemma}
\newdefinition{remark}{Remark}
\newproof{proof}{Proof}
%
%
%
\renewcommand{\vec}[1]{\mathbf{#1}}
%

\begin{frontmatter}

\title{Extensional Properties of Recurrent Neural Networks}

\author{Evgeny Dantsin}
\author{Alexander Wolpert} 

\affiliation{organization={Department of Computer Science, Roosevelt University},
            city={Chicago},
            state={IL},
            country={USA}}

\begin{abstract}
A property of a recurrent neural network (RNN) is called \emph{extensional} if, loosely speaking, it is a property of the function computed by the RNN rather than a property of the RNN algorithm. Many properties of interest in RNNs are extensional, for example, robustness against small changes of input or good clustering of inputs. Given an RNN, it is natural to ask whether it has such a property. We give a negative answer to the general question about testing extensional properties of RNNs. Namely, we prove a version of Rice's theorem for RNNs: any nontrivial extensional property of RNNs is undecidable.  
\end{abstract}

\end{frontmatter}

%
%

\section{Introduction} 
\label{sec:intro}


Given a recurrent neural network (RNN), how can we check whether it has a certain property? For example, how can we check whether an RNN runs in polynomial time? Or how can we check whether an RNN is robust against small changes of input? 

The first of these two properties is \emph{intensional}, which means that an RNN's running time is a characteristic of the RNN algorithm, not a characteristic of the function computed by the RNN. By contrast, the second property is \emph{extensional}, which means that the robustness is a characteristic of the function computed by an RNN, so all RNNs that compute this function are equally robust against small changes of input.

This paper is about extensional properties of RNNs and, more exactly, about testing such properties: we prove that any nontrivial extensional property of RNNs is undecidable. This is essentially an RNN version of Rice's theorem, the well known theorem from computability theory \cite{Ric53,Rog87}. Of course, this general negative result does not preclude designing algorithms that test undecidable properties for restricted classes of RNNs, like as the general negative results on program verification does not preclude verified software systems. 

To prove the undecidability result, we need a definition of RNNs as a model of computation. There are several such definitions in the literature, for example \cite{ZS15,GWR+16,FNW18,CS21} but, unfortunately, none of them works well for our two-fold purpose. On the one hand, the RNN model should be powerful enough to allow abstractions of RNNs used in practice. On the other hand, the model should be simple enough to allow precise analysis. Therefore, we had to define our own model. Namely, we define an RNN as an algorithm of a special type: it takes a sequence of vectors as input, processes them using operations on vectors and matrices, and outputs another sequence of vectors. A formal definition is given in terms of a variant of the abacus machine model \cite{Lam61,BBJ07}. 

The paper is organized as follows. Section~\ref{sec:prelim} provides the basics of computability theory that are needed for the next sections, including Rice's theorem. The notion of an RNN machine is defined in Section~\ref{sec-rnn}; this section also gives a formal definition of extensional properties of RNN machines. A version of Rice's theorem for RNN machines is proved in Section~\ref{sec:rnn-rice}. Section~\ref{sec:properties} gives examples of extensional properties of RNNs; it describes two groups of such properties, namely, properties of RNNs used for prediction and properties of RNNs used for clustering. 

%
%

\section{Preliminaries: Rice's Theorem}
\label{sec:prelim}


The terminology and notation in this section are mostly from \cite{Rog87,Soa16}.   

\paragraph{Partial computable functions}
A \emph{partial function} from a set $X$ to a set $Y$, denoted by $f: X \rightharpoonup Y$, is a function from $D$ to $Y$ where $D \subseteq X$. Thus, a function (also called a \emph{total function}) is a special case of partial functions that occurs when $D=X$. If $f$ and $g$ are partial functions, the equality $f(x)=g(x)$ means that either both sides of the equality are defined and are equal, or else both are undefined.

We identify algorithms with Turing machines (TMs). More exactly, we consider the variant of TMs described in \cite{Rog87,Soa16}: each such TM computes some partial function from $\mathbb{N}$ to $\mathbb{N}$, where $\mathbb{N}$ denotes the set $\{0, 1, 2, 3, \ldots\}$ of natural numbers. A partial function $f: \mathbb{N} \rightharpoonup \mathbb{N}$ is called \emph{computable} if $f$ is computed by some TM. 

The set of all partial functions from $\mathbb{N}$ to $\mathbb{N}$ is denoted by $\mathcal{N}$. The set of all partial computable functions is denoted by $\mathcal{C}$.

\paragraph{Descriptions of TMs}
Every TM can be encoded by a natural number called the \emph{description} of this TM. Moreover, we can require that every natural number is the description of some TM. Assuming that such a method of encoding is fixed, we write $M_e$ to denote the TM whose description is $e$. 

\paragraph{Canonical numbering}
The descriptions of TMs induce the following function $\varphi$ from $\mathbb{N}$ onto $\mathcal{C}$: every natural number $e$ is mapped to the partial computable function computed by the TM $M_e$. The function $\varphi$ is called the \emph{canonical numbering of $\mathcal{C}$}; the value of $\varphi$ at $e$ is usually written as $\varphi_e$ instead of $\varphi(e)$. The natural number $e$ is called an \emph{index} of the partial computable function $\varphi_e$. Note that each partial computable function is computed by infinitely many TMs and, therefore, it has infinitely many indices. 

\paragraph{Decidable and undecidable sets}
A subset of $\mathbb{N}$ is called \emph{decidable} if its characteristic function is computable; otherwise the subset is called \emph{undecidable}. A subset $S \subseteq \mathbb{N}$ is often identified with the following decision problem: given a natural number, does it belong to $S$? Depending on whether $S$ is decidable or undecidable, this decision problem is referred to as decidable or undecidable respectively. 

\paragraph{Rice's theorem}
Many decision problems about TMs are known to be undecidable, for example, the \emph{halting problem} (given  natural numbers $e$ and $x$, does the TM $M_e$ halts on $x$?). This problem is stated as a question about TMs but, essentially, this question is about partial computable functions: is $\varphi_e$ defined on $x$? Rice's theorem generalizes undecidability results for all properties of TMs that are, in fact, properties of the corresponding partial computable functions. Loosely speaking, Rice's theorem states that it is not possible to determine whether a TM has such a property. 

Formally, we define a \emph{property of TMs} as a subset $P \subseteq \mathbb{N}$, meaning that $P$ is the set of the descriptions of all TMs that have this property. A property $P$ is called \emph{nontrivial} if $P \neq \emptyset$ and $P \neq \mathbb{N}$. A property $P$ is called \emph{extensional} if for any two indices $m$ and $n$ of the same partial computable function, either both $m$ and $n$ are in $P$ or neither of them is in $P$: 
$$
\varphi_m = \varphi_{n} \ \, \Rightarrow \ \, (m \in P \ \Leftrightarrow \ n \in P) .
$$
That is, $P$ is extensional if this set is ``closed'' in the following sense: for every partial computable function $f$, if at least one index of $f$ is in $P$, then all other indices of $f$ are in $P$ too.

\begin{theorem}[Rice's theorem]
\label{th:rice}
Any nontrivial extensional property of TMs is undecidable.
\end{theorem}

Rice's theorem can be proved either using the recursion theorem or using undecidability results, for example, using the undecidability of the halting problem.


\paragraph{Rice's theorem for other models of computation}
Rice's theorem is stated above in terms of TMs, what about other models of computation? The class of models of computation for which Rice's theorem holds can be characterized in terms of numberings induced by the models. 

Let $\alpha: \mathbb{N} \to \mathcal{N}$ be a function that assigns the natural numbers to some partial functions from $\mathbb{N}$ to $\mathcal{N}$. Like the notation for the canonical numbering $\varphi$ above, we write $\alpha_e$ instead of $\alpha(e)$. This function $\alpha$ is called an \emph{acceptable numbering of $\mathcal{C}$} if there are computable functions $f: \mathbb{N} \to \mathbb{N}$ and $g: \mathbb{N} \to \mathbb{N}$ such that the following holds for all numbers $e \in \mathbb{N}$:
\begin{itemize}
\item $\alpha_e = \varphi_{f(e)}$ (computability: the partial function $\alpha_e$ is indeed computable by some TM);
\item $\varphi_e = \alpha_{g(e)}$ (completeness: $\alpha$ is indeed a function \emph{onto} $\mathcal{C}$).
\end{itemize}

\begin{theorem}[Rice's theorem in general form]
\label{th:general-rice}
Let $\alpha$ be an acceptable numbering of $\mathcal{C}$. Let $P$ be a subset of $\mathbb{N}$ such that
\begin{itemize}
\item $P$ is nontrivial: $P \neq \emptyset$ and $P \neq \mathbb{N}$;
\item $P$ is extensional with respect to $\alpha$: for all $m, n \in \mathbb{N}$, if $\alpha_m = \alpha_n$ then either both $m$ and $n$ are in $P$ or neither of them is in $P$.
\end{itemize}
Then $P$ is undecidable.
\end{theorem}
The meaning of $P$ in this theorem is the same as the meaning of $P$ in Theorem~\ref{th:rice}: given some model of computation and given some property of its computational devices, $P$ is the set of the descriptions of all computational devices that share this property. 

Note that common models of computation, like recursive functions or register machines, induce acceptable numberings of $\mathcal{C}$. Therefore, by Theorem~\ref{th:general-rice}, any nontrivial extensional property of their computational devices is undecidable. In Section~\ref{sec:rnn-rice}, we use this theorem to prove that any nontrivial extensional property of RNN machines (defined in Section~\ref{sec-rnn}) is undecidable.


\paragraph{Other notation used in the paper}
The set of rational numbers is denoted by $\mathbb{Q}$. We use notation like $\vec{v} = (r_1, \ldots, r_n)$ to denote elements of $\mathbb{Q}^n$. Any such sequence $\vec{v}$ is called a \emph{rational vector}, or just a \emph{vector}. We also consider finite sequences $\vec{v}_1, \ldots, \vec{v}_k$ of vectors, where the vectors may have different numbers of elements. The set of all such sequences is denoted by $\mathcal{V}$. 

%
%

\section{RNN Machines}
\label{sec-rnn}


In the literature, the term ``recurrent neural network'' is rather an umbrella term with different meanings that depend on what aspect of RNNs is analyzed. In this paper, we think of an RNN as an algorithm that takes a sequence $\vec{x}_1, \ldots, \vec{x}_s$ of vectors as input, processes them in a special way, and outputs a sequence $\vec{y}_1, \ldots, \vec{y}_t$ of vectors, see for example the textbooks \cite{Agg23,Dro22}. How is the input sequence transformed into the output one? 

Consider an RNN that takes a sequence from $\mathcal{V}$ as input (typically, but not necessarily, all input vectors have the same number of elements). They are processed using compositions of vector operations of the following three types: 
\begin{itemize}
\item Addition $\vec{u} = \vec{v}+\vec{w}$, where $\vec{u}, \vec{v}, \vec{w} \in \mathbb{Q}^k$.
\item Multiplication $\vec{u} = \vec{A} \vec{v}$, where $\vec{u} \in \mathbb{Q}^m$ and $\vec{v} \in \mathbb{Q}^n$ are vectors and $\vec{A}$ is an $m \times n$ matrix over $\mathbb{Q}$. 
\item Nonlinear operation $\vec{u} = f(\vec{v})$, where $f$ is a nonlinear activation function such as, for example, the ReLU activation function that replaces every rational number $r$ in the vector $\vec{v}$ with the number $\max(0,r)$.
\end{itemize}
Given a sequence $\vec{x}_1, \ldots, \vec{x}_s$ of input vectors, the computation of the RNN on this sequence can be represented as a sequence of vectors 
\begin{equation}
\label{eq:sequence}
\vec{v}_1, \ldots, \vec{v}_l
\end{equation}
where each vector $\vec{v}_i$ is obtained by applying one of the operations above. The application uses either input vectors, or vectors $\vec{v}_j$ with $j < i$ obtained earlier, or vectors ``hard-wired'' into the RNN. If a matrix is used in the application, the matrix is ``hard-wired'' too. Some vectors in (\ref{eq:sequence}) are returned as output vectors; the others are just needed for producing the output.

In addition to the operations above, the RNN can also use control operations needed for deciding when sequence (\ref{eq:sequence}) ends. Note that the RNN is not required to halt on all inputs, and it can produce vectors infinitely long as, for example, in \cite{ZYC+12,HBB21}. 


\paragraph{Abacus machines}
We are going to define RNN machines as register machines, more exactly, as \emph{abacus machines} \cite{Lam61} equipped with additional operations. Therefore, we first briefly describe a variant of the abacus machine model, see \cite{BBJ07} for details.

An \emph{abacus} has an infinite array of registers $R_0, R_1, R_2, \ldots$, where each register either is empty or stores a natural number. The number contained in $R_i$ is denoted by $[i]$. The machine also has a finite sequence of instructions $I_0, I_1, \ldots, I_l$ that can be \emph{conditional} or \emph{non-conditional}. There are two non-conditional instructions:
\begin{itemize}
\item $\mathtt{Halt}$. This is a command to halt.
\item $\mathtt{Zero}(i)$. This is a command to put $0$ in $R_i$ (if $R_i$ was nonempty, $[i]$ is replaced with $0$). After putting $0$ in $R_i$, the machine goes to the next instruction in the sequence $I_0, I_1, \ldots, I_l$. 
\end{itemize}
Any conditional instruction specifies a condition on registers' contents. If the condition is true, then the machine performs a specified action and goes to the next instruction in the sequence $I_0, I_1, \ldots, I_l$. Otherwise, the machine jumps to a specified instruction without performing the action. There are two conditional instructions:
\begin{itemize}
\item $\mathtt{Increase}(i, q)$. If $R_i$ contains a number, increase $[i]$ by $1$. Otherwise, jump to $I_q$. 
\item $\mathtt{Decrease}(i, q)$. If $R_i$ contains a positive number, decrease $[i]$ by $1$. Otherwise, jump to $I_q$. 
\end{itemize}

Let $f$ be a partial function from $\mathbb{N}$ to $\mathbb{N}$. We say that an abacus $A$ \emph{computes} $f$ if for every $x \in \mathbb{N}$, 
\begin{itemize} 
\item at the start, $A$ performs $I_0$, the register $R_0$ contains $x$, and all other registers are empty;
\item $A$ halts if and only if $f$ is defined on $x$; when halting, $R_0$ contains $f(x)$.
\end{itemize}
It is known that the abacus machine model is equivalent to the TM model: every partial function from $\mathbb{N}$ to $\mathbb{N}$ is computable by some abacus if and only if it is computable by some TM \cite{BBJ07}.


\paragraph{Definition of RNN machines}
Like abacus machines, an \emph{RNN machine} has registers $R_0, R_1, R_2, \ldots$ and a sequence of instructions $I_0, I_1, \ldots, I_l$. Each register either is empty or contains a natural number. In addition to the four abacus instructions described above, RNN machines have instructions of another type that operate on rational vectors. Before describing these instructions, we need to say how rational numbers and vectors are stored in registers.

Each rational number $r$ is represented in the form $r = (-1)^s (a/b)$, where $a, b \in \mathbb{N}$, $b \neq 0$, and $s$ is $0$ or $1$. We say that $r$ is \emph{stored in $R_i$, $R_{i+1}$, and $R_{i+2}$} if these registers store $s$, $a$, and $b$ respectively. That is, a rational number is stored in three consecutive registers. 

Vectors are stored in a similar manner. Let $\vec{v} = (r_1, \ldots, r_n)$ be a vector. We say that $\vec{v}$ is \emph{stored beginning with $R_i$} if $R_i$ stores $n$ and the next $3n$ registers store $r_1, \ldots, r_n$. That is, if $r_k = (-1)^s (a/b)$ then 
$$
[i+3k-2]=s, \ [i+3k-1]=a, \ [i+3k]=b . 
$$
If $\vec{v}$ is a vector stored beginning with $R_i$, we also say that $R_i$ \emph{determines} $\vec{v}$.
 
RNN machines have instructions for vector addition, multiplication of a vector by a matrix, and nonlinear activation (for certainty, we consider the ReLU activation, but we could use other nonlinear activation functions as well). In the multiplication instruction, a matrix of size $m \times n$ is represented as a sequence of $m$ vectors, each with $n$ elements.  
\begin{itemize}
\item $\mathtt{Add}(i, j, k, q)$. If $R_i$ determines a vector $(a_1, \ldots, a_n)$ and $R_j$ determines a vector $(b_1, \ldots, b_n)$, then store the vector $(a_1 + b_1, \ldots, a_n + b_n)$ beginning with $R_k$. Otherwise, jump to $I_q$.
\item $\mathtt{Multiply}(i_1, \ldots, i_m, j, k, q)$. The condition in this instruction is that (1) the registers $R_{i_1}, \ldots, R_{i_m}$ determine vectors $\vec{u}_1, \ldots, \vec{u}_m$ with $n$ elements in each, and (2) the register $R_j$ determines a vector $\vec{v}$ with $n$ elements too. If the condition is true, store the product $\vec{A} \vec{v}$ beginning with $R_k$, where $\vec{A}$ is the $m \times n$ matrix in which the rows are the vectors $\vec{u}_1, \ldots, \vec{u}_m$. Otherwise, jump to $I_q$.
\item $\mathtt{NonLinear}(i, j, q)$. If $R_i$ determines a vector $\vec{v} = (r_1, \ldots, r_n)$, store the vector $\vec{v}' = (r'_1, \ldots, r'_n)$ beginning with $R_j$, where $r'_k = \max(0, r_k)$ for $k = 1, \ldots, n$. Otherwise, jump to $I_q$.
\end{itemize}

Sequences of vectors are stored in RNN machines as follows. Let $x = (\vec{x}_1, \ldots, \vec{x}_s)$ be a sequence of vectors, where each vector $\vec{x}_i$ has $n_i$ elements. We say that $x$ is \emph{stored beginning with $R_i$} if $R_i$ stores $s$ and the next 
$$
s + 3(n_1 + \ldots + n_s)
$$
registers store the vectors $\vec{x}_1, \ldots, \vec{x}_s$ successively: $\vec{x}_1$, immediately followed by $\vec{x}_2$, etc. Thus, if $x$ is stored beginning with $R_i$, then $\vec{x}_i$ is stored beginning with $R_j$ where 
$$
j = 1 + i + 3(n_1 + \ldots + n_{i-1}).
$$.  

Let $f$ be a partial function from $\mathcal{V}$ to $\mathcal{V}$. We say that an RNN machine $N$ \emph{computes} $f$ on $x$ if for every input sequence $x = (\vec{x}_1, \ldots, \vec{x}_s)$ of vectors,  
\begin{itemize}
\item at the start, $N$ performs $I_0$, the sequence $x$ is stored beginning with $R_0$, and all other registers are empty;
\item $N$ halts if and only if $f$ is defined on $x$; when halting, $f(x)$ is stored beginning with $R_0$.
\end{itemize}


\paragraph{Extensional properties of RNN machines}
Like TMs, RNN machines can be encoded by natural numbers so that every natural number is the encoding of some RNN machine. We fix such a method of encoding, and we refer to the encoding of an RNN machine as the \emph{description} of this machine. The RNN machine whose description is $e$ is denoted by $N_e$. 

\begin{definition}[function $\psi$]
We define $\psi$ to be the function that maps every natural number $e$ to the partial function $\psi(e): \mathcal{V} \rightharpoonup \mathcal{V}$ computed by the RNN machine $N_e$. The partial function $\psi(e)$ is also written as $\psi_e$. 
\end{definition}
Thus, the family $\{\psi_e\}_{e \in \mathbb{N}}$ consists of those partial functions from $\mathcal{V}$ to $\mathcal{V}$ that can be computed by all RNN machines. From this point of view, the role of $\psi$ for RNN machines is the same as the role of $\varphi$ for Turing machines.

Like a property of TMs, a \emph{property of RNN machines} is defined as a subset $P \subseteq \mathbb{N}$, meaning that $P$ is the set of the descriptions of all RNN machines that share this property. We call $P$ \emph{nontrivial} if $P \neq \emptyset$ and $P \neq \mathbb{N}$. We call $P$ \emph{extensional} if the following holds for all $m, n \in \mathbb{N}$: 
$$
\psi_m = \psi_{n} \ \, \Rightarrow \ \, (m \in P \ \Leftrightarrow \ n \in P) .
$$
That is, if $m$ and $n$ are the descriptions of RNN machines that compute the same partial function from $\mathcal{V}$ to $\mathcal{V}$, then either both $m$ and $n$ are in $P$ or neither of them is in $P$. 

%
%

\section{Rice's Theorem for RNN Machines}
\label{sec:rnn-rice}


In this section, we show that Rice's theorem holds for RNN machines as well as for TMs. To prove the ``RNN version'' of Rice's theorem, we use RNN machines to build an acceptable numbering of $\mathcal{C}$.

\paragraph{Bijection $b$}
Since the set $\mathcal{V}$ (the set of all finite sequences of rational vectors) is countable, there is a computable bijection from $\mathcal{V}$ to $\mathbb{N}$. We fix such a bijection and denote it by $b$. The inverse of $b$ (computable as well) is denoted by $b^{-1}$.

\paragraph{Function $\psi'$}
We define $\psi'$ to be the function that maps every natural number $e$ to the composition $b \circ \psi_e \circ b^{-1}$. That is, the value $\psi'(e)$, also denoted by $\psi'_e$, is a partial function from $\mathbb{N}$ to $\mathbb{N}$ composed of three actions:
\begin{itemize}
\item first, the bijection $b^{-1}$ converts an input natural number to a sequence $\vec{x}_1, \ldots, \vec{x}_s$ of vectors;
\item then $\psi_e$ maps this sequence to another sequence $\vec{y}_1, \ldots, \vec{y}_t$ of vectors;
\item finally, $b$ converts the sequence $\vec{y}_1, \ldots, \vec{y}_t$ to an output natural number.
\end{itemize}

\paragraph{Acceptable numbering}
The function $\psi'$ defined above is an acceptable numbering of $\mathcal{C}$, which is proved in the two lemmas below.

\begin{lemma}[computability of $\psi'$]
\label{th:computability}
There exists a computable function $f: \mathbb{N} \to \mathbb{N}$ such that $\psi'_e = \varphi_{f(e)}$ for all $e \in \mathbb{N}$.
\end{lemma}

\begin{proof}
We first show that every natural number $e$ can be transformed to an abacus $A_e$ that computes the partial function $b \circ \psi_e \circ b^{-1}$. On every given input $n \in \mathbb{N}$, this abacus $A$ simulates three successive computations performed by abacus machines $A_1$, $A_2$, and $A_3$ where 
\begin{itemize}
\item $A_1$ computes the bijection $b^{-1}$ on $n$;
\item $A_2$ computes the partial function $\psi_e$ on $b^{-1}(n)$;
\item $A_3$ computes the bijection $b$ on the result of the previous computation.
\end{itemize}
The abacus machines $A_1$ and $A_3$ exist because (1) the fixed bijection $b$ and its inverse $b^{-1}$ are computable, and (2) the abacus machine model is equivalent to the TM machine model. The abacus $A_2$ exists as well because the instructions $\mathtt{Add}$, $\mathtt{Multiply}$, and $\mathtt{NonLinear}$ can be simulated by an abacus in an obvious way. 

Note that the abacus machines $A_1$ and $A_3$ do not depend on $e$. The abacus $A_2$ depends on $e$, and this abacus can be built from the RNN machine $N_e$ effectively. Therefore, the abacus $A$ that simulates the successive computations of $A_1$, $A_2$, and $A_3$ can be obtained from $e$ as well. Therefore, a TM equivalent to $A$ can be obtained from $e$, which gives the required function $f$. 
\end{proof}

\begin{lemma}[completeness of $\psi'$]
\label{th:completeness}
There exists a computable function $g: \mathbb{N} \to \mathbb{N}$ such that  $\varphi_e = \psi'_{g(e)}$ for all $e \in \mathbb{N}$.
\end{lemma}

\begin{proof}
Since the TM model and the abacus machine model are equivalent, it suffices to show that for every abacus $A$, one can build an RNN machine $N$ such that for every $x, y \in \mathbb{N}$,
$$
A(x)=y \ \ \Leftrightarrow \ \ N\left(b^{-1}(x)\right)=b^{-1}(y) .
$$ 
Such an RNN machine $N$ works as follows. First, $N$ takes the input sequence $b^{-1}(x)$ and converts it back to $x$. Then $N$ uses the instructions of $A$ to compute the number $y$. Finally, $N$ converts $y$ to the sequence $b^{-1}(y)$ and returns this sequence as output.
\end{proof}

\begin{theorem}[Rice's theorem for RNN machines]
\label{th:RNN-rice}
Any nontrivial extensional property of RNN machines is undecidable.
\end{theorem}

\begin{proof}
Let $P$ be a nontrivial extensional property of RNN machines. The extensionality of $P$ means that for all $m, n \in \mathbb{N}$,  
$$
\psi_m = \psi_{n} \ \, \Rightarrow \ \, (m \in P \ \Leftrightarrow \ n \in P) .
$$ 
By definition of $\psi'$, we have $\psi'_e = b \circ \psi_e \circ b^{-1}$ and, therefore,  
$$
\psi'_m = \psi'_{n} \ \, \Rightarrow \ \, (m \in P \ \Leftrightarrow \ n \in P) .
$$
The latter implication states that the set $P$ is extensional with respect to $\psi'$. Also, Lemmas~\ref{th:computability} and \ref{th:completeness} show that $\psi'$ is an acceptable numbering for $\mathcal{C}$. Therefore, we can apply Theorem~\ref{th:general-rice} taking $\psi'$ as $\alpha$. By this theorem, $P$ is undecidable.
\end{proof}

%
%

\section{Examples of Extensional Properties} 
\label{sec:properties}


We illustrate the notion of extensional properties of RNNs by describing two groups of such properties. One group is important for applications where RNNs are expected to be robust against small changes of input. The other group includes properties expected from RNNs that make clustering of their inputs.  


\subsection{Robustness Properties} 
\label{sec:robust}


A typical example of robustness is the property of a prediction system to make ``close'' predictions when input data ``slightly differ'' from each other \cite{DJS+21,ZZZ+23,CDK+24}. How can we define robustness properties in the RNN setting?

It is natural to formalize robustness in terms of continuity of functions defined on metric spaces. Informally speaking, an RNN is robust if the function computed by this RNN is continuous. More formally, let $N_e$ be an RNN machine that takes inputs from a set $X$ and returns outputs that belong to a set $Y$. Let $d_\mathrm{in}$ be a metric on $X$ and let $d_\mathrm{out}$ be a metric on $Y$. We say that $N_e$ is \emph{robust} if the partial function $\psi_e$ is continuous at every input $x$ where it is defined: for every $\epsilon > 0$, there exists a $\delta > 0$ such that for all inputs $x'$ on which $N_e$ halts,
$$
d_\mathrm{in}(x, x') < \delta \ \ \Rightarrow \ \ d_\mathrm{out}(\psi_e(x), \psi_e(x')) < \epsilon .
$$ 
(A similar way to formalize robustness is to define it as uniform continuity or Lipschitz continuity.) Thus, a key point is how we define metric spaces $(X, d_\mathrm{in})$ and $(Y, d_\mathrm{out})$. We give three examples of such definitions. 

\paragraph{Robustness against small perturbations}
This form of robustness is for RNN machines that take as input sequences $x=(\vec{x}_1, \ldots, \vec{x}_s)$, where all vectors are from $\mathbb{Q}^m$, and output sequences $y=(\vec{y}_1, \ldots, \vec{y}_t)$, where all vectors are from $\mathbb{Q}^n$. The parameters $s, t, m, n$ are fixed for a given RNN machine. Metrics $d_\mathrm{in}$ on the inputs and $d_\mathrm{out}$ on the outputs are defined by  
\begin{eqnarray}
d_\mathrm{in}(x, x') & = & \max_{1 \le i \le s} \|\vec{x}_i - \vec{x}'_i\|; \\ 
\label{eq:out}
d_\mathrm{out}(y, y') & = & \max_{1 \le j \le t} \|\vec{y}_j - \vec{y}'_j\|. 
\end{eqnarray}

\paragraph{Robustness against deletions and insertions}
Consider an RNN machine for which (1) the set $X$ of inputs consists of arbitrarily long sequences of arbitrary vectors, and (2) the set $Y$ of outputs is defined as in the paragraph above, namely, as sequences of $t$ vectors from $\mathbb{Q}^n$. Suppose input sequences $x$ and $x'$ are obtained from each other by deletions and insertions of some vectors. If the number of deletions and insertions is small enough (compared with the lengths of $x$ and $x'$), then we expect that the output on $x$ differs from the output on $x'$ not too much. This form of robustness can be formalized using the following metrics. The metric $d_\mathrm{in}$ on inputs is based on longest common subsequences \cite{Bak09}: for sequences $x$ and $x'$ of lengths $|x|$ and $|x'|$ respectively, the distance between them is
$$
d_\mathrm{in}(x, x') = 1 - \frac{\ell(x, x')}{\max (|x|, |x'|)}
$$
where $\ell(x, x')$ is the length of a longest common subsequence of $x$ and $x'$. The metric $d_\mathrm{out}$ on outputs is defined as in (\ref{eq:out}).

\paragraph{Small perturbations combined with deletions and insertions}
This example, in a sense, combines the two examples above. Consider an RNN machine for which (1) the set $X$ of inputs consists of arbitrarily long sequences of vectors from $\mathbb{Q}^m$, and (2) the set $Y$ of outputs consists of arbitrarily long sequences of vectors from $\mathbb{Q}^n$. We can define a combined robustness of such RNNs against ``small'' changes of input by using the following \emph{global alignment metric}, also called the Needleman–Wunsch–Sellers metric \cite{DD09}. 

The set $X$ is equipped with three \emph{editing operations}: deletion, insertion, and replacement of a vector. Each application of these operations has its own \emph{cost}. Namely, the cost of any application of deletions and insertions is a fixed number, independent of the deleted or inserted vector. The cost of the replacement of a vector $\vec{u}$ with a vector $\vec{u}$ is defined as $\| \vec{u} - \vec{v}\|$. Then, for any inputs $x$ and $x'$, the distance $d_\mathrm{in}(x, x')$ is defined as the minimum total cost of a sequence of operations that transform $x$ into $x'$. This distance is a metric \cite{DD09}. The metric $d_\mathrm{out}$ is defined similarly. 


\subsection{Good Clustering Properties} 


The term ``clustering'' commonly refers to the following task: given a metric space, group its points into subsets, called \emph{clusters}, in such way that points in one cluster are closer to each other than to points in other clusters. Also, ``clustering'' often refers to the clusters obtained as a result of solving this task. In this section, we use ``clustering'' with the latter meaning. Namely, we consider the \emph{evaluation} task in which we are given a clustering and we need to evaluate how ``good'' the clustering is in the following sense:
\begin{itemize}
\item how well the clusters are separated from each other (\emph{cluster separation});
\item how close to each other the points within each cluster are (\emph{cluster connectedness}). 
\end{itemize}

\paragraph{Measures for the quality of a given clustering}
There are many ways to quantify cluster separation and cluster connectedness, and there are many ways to combine them. Correspondingly, there are many measures for the quality of a given clustering. As examples, we mention only two well known measures.

The first example is the \emph{Dunn index} \cite{Dun74}. Let $C$ be a clustering that consists of clusters $C_1, \ldots, C_k$. The \emph{Dunn index} of $C$, denoted by $D(C)$, is the ratio
$$
D(C) = \frac{\min_{\,1 \le i < j \le k} \, B(C_i, C_j)}{\max_{\, 1 \le i \le k} \, W(C_i)}
$$
where $B(C_i, C_j)$ is the \emph{between-cluster distance}, and $W(C_i)$ is the \emph{within-cluster distance}. Each of these two distances can be defined in different ways. For example, $B(C_i, C_j)$ can be the distance between the centroids of the clusters $C_i$ and $C_j$ or the minimum distance between any point from $C_i$ and any point from $C_j$. The within-cluster distance $W(C_i)$ can be the diameter of $C_i$ or the average distance between any pair of points in $C_i$. The Dunn index is a measure of the quality of $C$: the larger $D(C)$ is, the ``better'' the clustering $C$ is.

The second example is the \emph{silhouette index} \cite{Rou87} calculated as follows. As above, $C$ is a clustering that consists of clusters $C_1, \ldots, C_k$ (we assume that each of them has more than one point). For every point $x$ in the underlying metric space, we calculate two numbers $b(x)$ and $w(x)$:
\begin{itemize}
\item the between-cluster distance $b(x)$ is defined as the minimum distance from $x$ to all points in the other clusters;
\item the within-cluster distance $w(x)$ is defined as the average distance from $x$ to all other points in the same cluster.
\end{itemize}
Then the \emph{silhouette value} of $x$ is calculated as 
$$
s(x) = \frac{b(x)-a(x)}{\max\{a(x), b(x)\}}.
$$
Thus, the silhouette value of $x$ is between $-1$ and $1$. Closeness to $1$ means that $x$ is well clustered; closeness to $-1$ means that it would be better to move $x$ to another cluster. The \emph{silhouette index} of $C$, denoted by $S(C)$, is defined as the average silhouette value over all points. Like the Dunn index, the silhouette index shows how well the clustering $C$ is: the larger $S(C)$ is, the ``better'' $C$ is.


\paragraph{RNNs that make good clustering}
Consider an RNN machine $N_e$ and fix a metric on its inputs (Section~\ref{sec:robust} gives examples of such metrics). Assume that the set of outputs of $N_e$ is finite; let $\{y_1, \ldots, y_k\}$ be all outputs of $N_e$. Any such RNN machine makes clustering of its inputs in the following sense: $N_e$ groups its inputs into subsets $C_1, \ldots, C_k$ where each $C_i$ consists of all inputs on which $N_e$ returns $y_i$. The subsets $C_1, \ldots, C_k$ can be viewed as \emph{clusters}. 

Note that RNNs make clustering not only when they are developed and used for this purpose. Any RNN with a finite number of outputs, even if it is used for tasks different from clustering, makes clustering of its inputs. The question is how good this clustering is.   

We can formalize the property ``$N_e$ makes good clustering'' using numerous performance measures such as the Dunn index or the silhouette index sketched above. For example, taking silhouette index and choosing some threshold number $t$ close to $1$, this property can be formalized as $S(C) \ge t$. Obviously, all RNNs that compute $\psi_e$ make the same clustering as $N_e$, so this property is extensional.


\paragraph{Other tasks}
The examples above are extensional properties of RNNs used for prediction and clustering. What about other tasks? In fact, many properties of interest in many tasks are extensional, for example, in anomaly detection \cite{CYP+21}, in neural network quantization \cite{DAR+23}, in feature extraction (with RNN autoencoders) \cite{LPL23}, etc.

%

\bibliographystyle{elsarticle-num}
\bibliography{rice-refs}

\begin{thebibliography}{10}
\expandafter\ifx\csname url\endcsname\relax
  \def\url#1{\texttt{#1}}\fi
\expandafter\ifx\csname urlprefix\endcsname\relax\def\urlprefix{URL }\fi
\expandafter\ifx\csname href\endcsname\relax
  \def\href#1#2{#2} \def\path#1{#1}\fi

\bibitem{Ric53}
H.~G. Rice, Classes of recursively enumerable sets and their decision problems,
  Transactions of the American Mathematical Society 74~(2) (1953) 358--366.

\bibitem{Rog87}
H.~Rogers, Theory of recursive functions and effective computability (Reprint
  from 1967), {MIT} Press, 1987.

\bibitem{ZS15}
W.~Zaremba, I.~Sutskever, Reinforcement learning neural {T}uring machines, CoRR
  abs/1505.00521 (2015).

\bibitem{GWR+16}
A.~Graves, G.~Wayne, M.~Reynolds, T.~Harley, I.~Danihelka,
  A.~Grabska{-}Barwinska, S.~G. Colmenarejo, E.~Grefenstette, T.~Ramalho, J.~P.
  Agapiou, A.~P. Badia, K.~M. Hermann, Y.~Zwols, G.~Ostrovski, A.~Cain,
  H.~King, C.~Summerfield, P.~Blunsom, K.~Kavukcuoglu, D.~Hassabis, Hybrid
  computing using a neural network with dynamic external memory, Nature
  538~(7626) (2016) 471--476.

\bibitem{FNW18}
J.~Franke, J.~Jan~Niehues, A.~Waibel, Robust and scalable differentiable neural
  computer for question answering, in: Proceedings of the Workshop on Machine
  Reading for Question Answering, Association for Computational Linguistics,
  2018, pp. 47--59.

\bibitem{CS21}
S.~Chung, H.~Siegelmann, Turing completeness of bounded-precision recurrent
  neural networks, in: Advances in Neural Information Processing Systems 34:
  Annual Conference on Neural Information Processing Systems, {NeurIPS} 2021,
  2021, pp. 28431--28441.

\bibitem{Lam61}
J.~Lambek, How to program an infinite abacus, Canadian Mathematical Bulletin
  4~(3) (1961) 295–302.

\bibitem{BBJ07}
G.~S. Boolos, J.~P. Burgess, R.~C. Jeffrey, Computability and Logic, 5th
  Edition, Cambridge University Press, 2007.

\bibitem{Soa16}
R.~I. Soare, Turing Computability - Theory and Applications, Theory and
  Applications of Computability, Springer, 2016.

\bibitem{Agg23}
C.~Aggarwal, Neural Networks and Deep Learning: A Textbook, 2nd Edition,
  Springer International Publishing, 2023.

\bibitem{Dro22}
I.~Drori, The Science of Deep Learning, Cambridge University Press, 2022.

\bibitem{ZYC+12}
Y.~Zhang, Y.~Yang, B.~Cai, D.~Guo, Zhang neural network and its application to
  {N}ewton iteration for matrix square root estimation, Neural Computing and
  Applications 21~(3) (2012) 453--460.

\bibitem{HBB21}
H.~Hewamalage, C.~Bergmeir, K.~Bandara, Recurrent neural networks for time
  series forecasting: Current status and future directions, International
  Journal of Forecasting 37~(1) (2021) 388–427.

\bibitem{DJS+21}
T.~Du, S.~Ji, L.~Shen, Y.~Zhang, J.~Li, J.~Shi, C.~Fang, J.~Yin, R.~Beyah,
  T.~Wang, {Cert-RNN:} towards certifying the robustness of recurrent neural
  networks, in: {ACM} {SIGSAC} Conference on Computer and Communications
  Security, {CCS} 2021, {ACM}, 2021, pp. 516--534.

\bibitem{ZZZ+23}
X.~Zhang, C.~Zhong, J.~Zhang, T.~Wang, W.~W.~Y. Ng, Robust recurrent neural
  networks for time series forecasting, Neurocomputing 526 (2023) 143--157.

\bibitem{CDK+24}
M.~Casadio, T.~Dinkar, E.~Komendantskaya, L.~Arnaboldi, O.~Isac, M.~L. Daggitt,
  G.~Katz, V.~Rieser, O.~Lemon, {NLP} verification: Towards a general
  methodology for certifying robustness, CoRR abs/2403.10144 (2024).

\bibitem{Bak09}
D.~Bakkelund, An {LCS}-based string metric, Tech. rep., University of Oslo,
  Oslo, Norway (2009).

\bibitem{DD09}
M.~M. Deza, E.~Deza, Encyclopedia of distances, 2nd Edition, Springer Berlin
  Heidelberg, 2013.

\bibitem{Dun74}
J.~C. Dunn, Well-separated clusters and optimal fuzzy partitions, Journal of
  Cybernetics 4~(1) (1974) 95--104.

\bibitem{Rou87}
P.~J. Rousseeuw, Silhouettes: A graphical aid to the interpretation and
  validation of cluster analysis, Journal of Computational and Applied
  Mathematics 20 (1987) 53--65.

\bibitem{CYP+21}
K.~Choi, J.~Yi, C.~Park, S.~Yoon, Deep learning for anomaly detection in
  time-series data: Review, analysis, and guidelines, {IEEE} Access 9 (2021)
  120043--120065.

\bibitem{DAR+23}
A.~Durao, J.~Arrais, B.~Ribeiro, G.~Falcao, On the quantization of recurrent
  neural networks for smiles generation, in: {IEEE} International Conference on
  Acoustics, Speech and Signal Processing {ICASSP} 2023, {IEEE}, 2023, pp.
  1--5.

\bibitem{LPL23}
P.~Li, Y.~Pei, J.~Li, A comprehensive survey on design and application of
  autoencoder in deep learning, Applied Soft Computing 138 (2023) 110--176.

\end{thebibliography}

\end{document}